\documentclass[letterpaper]{article} 
\usepackage{aaai24}  
\usepackage{times}  
\usepackage{helvet}  
\usepackage{courier}  
\usepackage[hyphens]{url}  
\usepackage{graphicx} 
\usepackage{natbib}  
\usepackage{caption} 
\usepackage{algorithm}
\usepackage{algorithmic}

\usepackage{newfloat}
\usepackage{listings}
\DeclareCaptionStyle{ruled}{labelfont=normalfont,labelsep=colon,strut=off} 
\floatstyle{ruled}
\newfloat{listing}{tb}{lst}{}
\floatname{listing}{Listing}

\newcommand{\algo}{{\textsf{BOSouL}}}
\usepackage[english]{babel}
\usepackage{amsmath}
\usepackage{amsthm}
\usepackage{amssymb}
\theoremstyle{definition}
\newtheorem{definition}{Definition}
\theoremstyle{plain} 

\newtheorem*{theorem*}{Theorem}

\usepackage{xcolor}
\DeclareMathOperator*{\argmax}{argmax}

\title{Multiple-Source Localization from a Single-Snapshot Observation Using\\ Graph Bayesian Optimization}
\author{
    Zonghan Zhang,
    Zijian Zhang,
    Zhiqian Chen
}
\affiliations{

    Department of Computer Science and Engineering, Mississippi State University\\   
    zz239@msstate.edu, zz242@msstate.edu, zchen@cse.msstate.edu
}

\begin{document}

\maketitle

\begin{abstract}
Due to the significance of its various applications, source localization has garnered considerable attention as one of the most important means to confront diffusion hazards. Multi-source localization from a single-snapshot observation is especially relevant due to its prevalence. However, the inherent complexities of this problem, such as limited information, interactions among sources, and dependence on diffusion models, pose challenges to resolution. Current methods typically utilize heuristics and greedy selection, and they are usually bonded with one diffusion model. Consequently, their effectiveness is constrained.
To address these limitations, we propose a simulation-based method termed~\algo. Bayesian optimization (BO) is adopted to approximate the results for its sample efficiency. A surrogate function models uncertainty from the limited information. It takes sets of nodes as the input instead of individual nodes.~\algo~can incorporate any diffusion model in the data acquisition process through simulations. Empirical studies demonstrate that its performance is robust across graph structures and diffusion models. The code is available at https://github.com/XGraph-Team/BOSouL.
\end{abstract}

\section{Introduction}

In recent decades, the world has become more interconnected thanks to the emergence of various networks. Consequently, we have become more vulnerable to network diffusion risks such as the spread of rumors, influenza-like viruses, and smart grid failures~\cite{chowdhury2020joint,ozili2020spillover,amin2007preventing}. Source localization (SL), the reverse problem of information diffusion, has attracted significant attention from researchers as a necessary component of the confrontation against diffusion hazards~\cite{prakash2012spotting,zang2015locating,  wang2017multiple, zhu2017catch,dong2019multiple}. It holds importance across various application domains such as medicine, security, large interconnected networks, social networks, and more\cite{shelke2019source, li2007analysis}. Source localization can be leveraged to block negative influence (rumors and viruses), maintain infrastructure (power grid), determine accountability (propagators of rumors), and verify information reliability. For instance, negative news about an election candidate spread from the opponent's campaign office is less credible than one from a neutral third-party press.
\begin{figure}[t]
  \centering
  \includegraphics[width=0.98\linewidth]{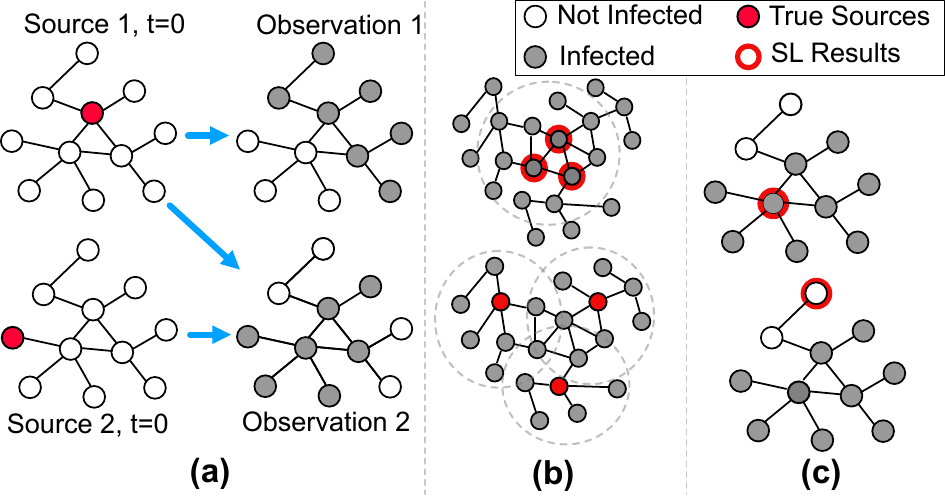}
    \vspace*{-2mm}
  \caption{The challenges faced by heuristic methods: (a) Sources and spreads have many-to-many relationships. (b) Greedy algorithms result in sub-optimal solutions (top). (c) Different diffusion models (SI at the top and SIR at the bottom) affect source localization results.
}
  \label{fig:1}
  \vspace*{-4mm}
\end{figure}

The SL problems can be divided into three classes based on network observation types: complete observation, monitor observation, and snapshot observation~\cite{shelke2019source}.
Among them, complete observation is rarely possible, given the massive scale of the real-world networks. Monitor observation, where part of the nodes are monitored for whether and when they are infected, is also not always available, especially for sudden outbreaks of unwanted propagation. Therefore, the most common scenario is single-snapshot observation, where only information about the network at a specific time is available (i.e., which nodes are infected at the snapshot). Early approach of single-snapshot source localization~\cite{prakash2012spotting} identified single sources of infectious diseases. Later, multi-source methods were proposed~\cite{wang2017multiple, zhu2017catch, dong2019multiple} to handle the wide existence of multiple diffusion sources in reality.

Although simulation-based methods seem to be legitimate to solve the SL problem, evaluating all node sets is infeasible due to the problem's combinatorial nature and the $\#P$-hardness of evaluating each set~\cite{kempe2003maximizing}.
Therefore, heuristic methods dominate the problem of multi-source localization from a single network snapshot~\cite{shah2010detecting,prakash2012spotting, zang2015locating, wang2017multiple, zhu2017catch, dong2019multiple, nie2019localizing}. 
However, this problem faces three major challenges that current heuristic methods cannot adequately address:
\textbf{(1) Many-to-many relationship between sources and spreads.}
Intuitively, one source set can result in various snapshots and vice versa. One snapshot may derive from multiple source sets with different probabilities. 
As shown in Figure~\ref{fig:1}(a), Source $1$ can lead to two totally different spreads, Observation $1$ and $2$. At the same time, Observation 2 can initiate from either Source $1$ or Source $2$.
It is difficult to properly model the many-to-many relationship using the limited information from a single snapshot due to the existence of this uncertainty~\cite{cai2018information}.
While heuristics efficiently approximate the localization, they only provide deterministic solutions without acknowledging the systematic uncertainty.
\textbf{(2) Complicated interactions among the sources.}
As a set of sources jointly spreads the influence through the network, considering their interrelationship is crucial for multi-source localization. Current heuristic methods approach the multi-source localization problem greedily~\cite{wang2017multiple,dong2019multiple,nie2019localizing}. 
The highest-scoring nodes are selected as sources after scoring each node based on corresponding heuristics. 
Nevertheless, ignoring the interactions among the sources has significant consequences. 
Figure~\ref{fig:1}(b) illustrates the imperfection of localizing sources greedily. Given an infected sub-graph, nodes selected by a greedy algorithm lie at the center of the sub-graph (top), while the true sources might locate at the centers of multiple hidden communities (bottom).
\textbf{(3) Heavy dependence on diffusion models.}
Diffusion problems involve three key entities: sources, diffusion model, and spread. Given any two, the third can be inferred. Therefore, source localization requires a diffusion model as one of the inputs. The same spread combined with different diffusion models will result in different source localization results. As demonstrated in Figure~\ref{fig:1}(c), the same observation leads to different results when combined with SI (top) and SIR (bottom). 
In the SI model, the sources must be within the infected sub-graph. However, in SIR, the sources are likely to have already recovered since they have a higher recovery possibility than other nodes. 
Nearly all heuristic methods are designed for one specific diffusion model~\cite{prakash2012spotting,zang2015locating}. Some claim to be model-free but make assumptions that implicitly constrain the diffusion model. For instance, LPSI~\cite{wang2017multiple} assumes that the sources are currently infected, indicating the diffusion model is like SI.

To address these issues, we propose a relatively efficient simulation-based method termed~\algo. Specifically, we evaluate the candidate source set's likelihood of being the true source set via simulations. Bayesian optimization is employed for its sample efficiency to reduce the number of simulations. A surrogate function mimics the relationship between the candidate set and its likelihood. Sampling through clustering is conducted to guarantee the sampled instances in each Bayesian optimization iteration are evenly distributed in the search space.
Our primary contributions include:
\begin{itemize}
\item \textbf{We propose an efficient simulation-based method utilizing Bayesian optimization}.~\algo~generates a Gaussian process that captures the uncertainty in the relationship between the set of sources and the observed snapshot. Furthermore, the Bayesian optimization paradigm significantly reduces the number of simulations, making the time cost of the algorithm acceptable. 

\item \textbf{The multiple sources are evaluated as a set instead of individually.} Thus, the interrelationship among the sources is included in~\algo~and its corresponding model. Comparing the greedy methods, the performance of~\algo~is more robust across different networks.

\item \textbf{Our method can be combined with any diffusion model.} As long as we have a diffusion model that can well capture the diffusion pattern, we can adopt that model in the simulations to build the relationship between a source set and an observation.

\item \textbf{We provide time complexity analysis. Extensive empirical experiments are conducted.} Real-world and synthetic datasets are employed to demonstrate~\algo's superior performance. It is also displayed that~\algo~scales as well as most of the baselines, and the runtime is reasonably acceptable.

\end{itemize}
\vspace{-2mm}

\section{Related Work}


\noindent\textbf{Source localization}, which aims to infer the origins of diffusion processes on networks given the diffused observation, has significant applications such as identifying rumor sources~\cite{shelke2019source} and finding 'patient zero' in a pandemic~\cite{scarpino2019predictability}. It has attracted growing research interest in recent years~\cite{shah2010detecting,prakash2012spotting, zang2015locating, wang2017multiple, zhu2017catch, dong2019multiple, nie2019localizing}. 
Diffusion studies have presented multiple diffusion models, such as epidemic models like susceptible-infected (SI), susceptible-infected-recovered (SIR), and susceptible-infected-susceptible (SIS)~\cite{brauer2019mathematical} and influence models like independent cascade (IC) and linear threshold (LT)~\cite{kempe2003maximizing}. However, early works focused on locating single sources under prescribed diffusion models. For instance, a few methods are designed specifically for SI model~\cite{prakash2012spotting, shah2010detecting, nie2019localizing}, and some others are designed for SIR~\cite{zhu2017catch}. Wang et al.~\cite{wang2017multiple} proposed a label propagation method named LPSI to detect multiple sources without knowing the underlying propagation model. Dong et al.~\cite{dong2019multiple} further enhanced LPSI by incorporating graph neural networks. However, since LPSI and its variants assume the sources are in the infected sub-graph, they implicitly suggest an SI-like diffusion model. Generally speaking, the current methods are bonded with certain diffusion models and lack generalizability. Also, most of the methods are simply greedily select sources based on single-source localization algorithms. The interrelationship among the sources is overlooked or intentionally ignored.

\noindent\textbf{Bayesian Optimization} is an approach for optimizing black-box functions that are expensive to evaluate. It constructs a probabilistic model of the objective function and uses this model to determine promising candidates to evaluate next~\cite{frazier2018tutorial}. Bayesian optimization was first proposed by Mockus et al. ~\cite{mockus1998application} and has since become a popular methodology for hyperparameter tuning and optimization of complex simulations and models~\cite{snoek2012practical}. The key idea is to leverage Bayesian probability theory to model uncertainty about the objective function. A prior distribution is placed over the space of functions, often a Gaussian process, which is updated as observations are made. An acquisition function then uses this model to determine the next evaluation point by balancing exploration and exploitation. Some common acquisition functions include expected improvement, knowledge gradient, and upper confidence bound~\cite{shahriari2015taking}. There has been much work extending Bayesian optimization to handle constraints~\cite{gelbart2014bayesian}, parallel evaluations~\cite{gonzalez2016batch}, and high dimensions~\cite{de2013bayesian}. Overall, Bayesian optimization provides an elegant and principled approach to sample-efficient optimization of black-box functions.
Bayesian optimization over a graph search space has emerged in the past decades. However, most of the works focus on node-level tasks and thus develop specific kernels for node smoothing~\cite{ng2018bayesian,oh2019combinatorial,walker2019graph,opolka2020graph,borovitskiy2021matern,opolka2022adaptive}. These works, while related, deal with a different task and the methods cannot be applied on our problem.
\vspace{-2mm}

\section{Method}
\label{method}

We propose a Bayesian Optimization for Source Localization (\algo) approach that combines the Bayesian optimization paradigm with simulations to enable more precise inference of source sets from single-snapshot observations.
\vspace{-4mm}

\subsection{Problem Formulation}

\begin{definition}[Single-snapshot Multi-Source Localization]
\textit{
Given 
\textbf{(1)} a size-$N$ graph $G(V,E)$ where $V$ and $E$ represent vertices and edges, respectively. 
\textbf{(2)} one single observation of propagation snapshot represented by a vector $o^* = \{0, 1\}^{N}$, where $1$ means the node is infected and $0$ means otherwise, 
\textbf{(3)} the underlying diffusion model $d$, and 
\textbf{(4)} the source $s= \{0, 1\}^{N}$ with its cardinality $|s|>1$. Note that any source node is not in the neighborhood of any other source node. The objective is to find the optimal $s$ that maximizes its conditional probability $\mathbf{P}$:
\vspace{-2mm}
\begin{equation}\small
s = \arg\max_{s} \mathbf{P}(s|o^*, G, d),
\label{eq:posterior}
\vspace{-2mm}
\end{equation}
where $\mathbf{P}$ is a conditional probability of $s$ given $o^*$, $G$ and $d$.}
\end{definition}

The difficulty presented by the aforementioned task is that the true source node set is unknown and cannot be retrieved during the source localization procedure. Therefore, it is impossible to compute the distance between the predicted and actual sources during the learning procedure. 
Extending Equation \ref{eq:posterior} with Bayes rule, we have:
\vspace{-2mm}
\begin{equation*}
\small
\mathbf{P}(s|o^*, G, d)=\frac{\mathbf{P}(o^*|s, G, d)\mathbf{P}(s)}{\mathbf{P}(o^*)}\sim \mathbf{P}(o^*|s, G, d),
    \vspace{-2mm}
\end{equation*}since no assumption is applied for $P(o)$ and $P(s)$, which will be set to uniform distribution as the prior probability. Then the task is changed to 
\vspace{-2mm}
\begin{equation*}
\small
    s = \arg\max_{s} \mathbf{P}(o^*|s, G, d).
    \vspace{-2mm}
\end{equation*}
The probability of a candidate source set $s$ is evaluated with the similarity between its simulated propagation spread $o$ and the observed snapshot $o^*$. So the estimated source is:
\vspace{-2mm}
\begin{equation}
\small
        \hat{s} = \arg\max_{s} \mathbf{P}(o^*|s, G, d)\sim\arg\max_{s} \textsc{SIM}(o, o^*; G, d).
        \label{eq:obj}
        \vspace{-2mm}
\end{equation}where $o$ is the simulation result of $s$ on $G$ with d. $\textsc{SIM}$ denotes a similarity metric between a pair of observations.
\vspace{-2mm}

\subsection{The Proposed Method:~\algo}
\begin{figure*}[htpb!]
  \centering
  \includegraphics[width=0.94\textwidth]{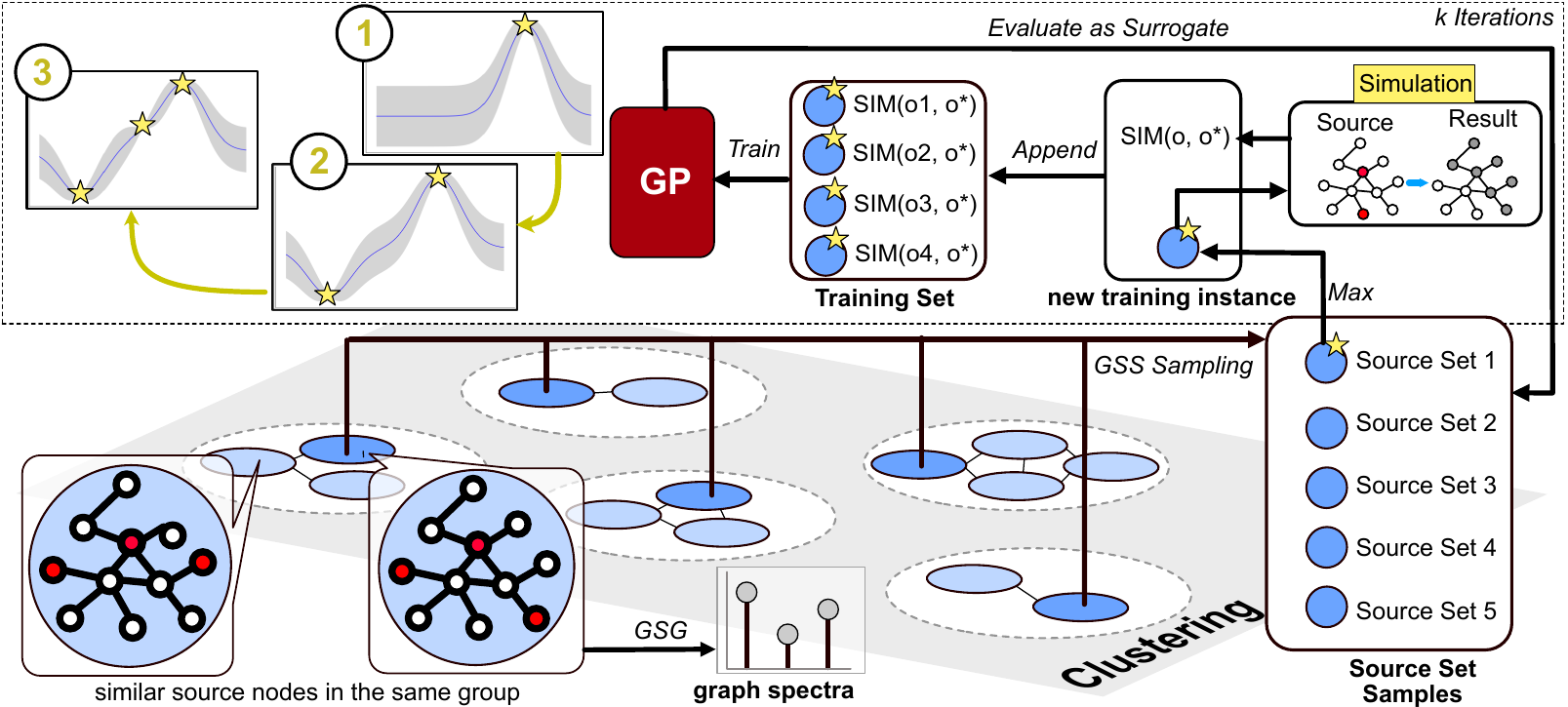}
  \vspace{-2mm}
  \caption{Illustration of the proposed $\algo$. Based on their graph Fourier representations, all possible candidates for source node sets are clustered in the bottom half of this figure. As training sets for GP, the model samples instances from each group utilizing GSS. In the upper half, one optimal instance is chosen by the surrogate GP (Max), and its likelihood of being a genuine source is determined through simulation. Sampling and training will be performed iteratively, and in each iteration, the GP model will be updated from a prior to a posterior (for example, from 1 to 2 on the upper left).}
  \label{fig:2}
  \vspace*{-4mm}
\end{figure*}

\subsubsection{Overview.}
As shown in Figure~\ref{fig:2}, we propose a Bayesian optimization-based learning framework for one-shot multi-source localization.
First, a kernel for the Gaussian Process (GP) is devised to measure the distance between source nodes, and then its validity is proven through theoretical analysis. Using the kernel, the output of the GP model is derived as a surrogate to predict the probability of a given source node set.
Next, we initialize GP with multiple actual simulations and select sites with expected improvement (EI) iteratively. 
Ultimately, the optimal solution is determined by traversing all candidates within the designated range.

\subsubsection{Gaussian Process Design.}
Consider a graph with $N$ nodes. Node sets are typically associated with a binary vector, labeled as $1$ if they are sources and $0$ otherwise. This vector is represented with $s= \{0, 1\}^{N}$. With $k$ sources, the total possible source configurations is $\binom{N}{k}$. Recognizing that not all nodes are equally significant in diffusion, like major cities in transport networks or key influencers in social networks, we focus on the top $a$ nodes by degree. This reduces potential source combinations to $\binom{a}{k}$ where $a \ll N$.

Meanwhile, $s$ is only a one-hot vector and lacks graph structure information. To illustrate, consider two 3-node sets: one original and the other formed by shifting each node by one hop based on the original one. 
Although the final observations are anticipated to be similar for these two sets, their similarity with the binary representations is quite low (0 in this case). 
This binary representation inadequately characterizes the similarity between two sets of nodes and violates the smoothness assumption imposed by the Gaussian process.
Previous work for graph kernels prioritizes structural comparisons, often ignoring attributes over the graphs \cite{vishwanathan2010graph,kriege2020survey,nikolentzos2021graph,siglidis2020grakel}. 
To overcome this constraint, we propose the introduction of a novel kernel that effectively combines structure information with theoretical validity.
First, the source vector \( s \)) is transformed into its Fourier counterpart \( \Tilde{s} \) such that:
\vspace{-2mm}
\begin{equation}\small
    \Tilde{s} = U^{\top}s, \quad \Tilde{s}(i) = \mathop{\Sigma}\limits^n_{i=1} s_i U^{\top}(i),
    \label{eq:fourier_transform}
    \vspace{-2mm}
\end{equation}where $U^{\top}$ is the inverse eigenvectors of the graph Laplacian and serves a graph Fourier transformer. 
Combining the graph Fourier transform and RBF kernel, we have a new kernel termed as graph spectral Gaussian (GSG) kernel:
\vspace{-2mm}
\begin{equation}\small
    \mathcal{K}(x,x';\mathit{l}) = exp(-\frac{||U^\top x-U^\top x'||^2}{2\mathit{l}^2}),
    \label{eq:kernel}
    \vspace{-2mm}
\end{equation} where $\mathit{l}$ is a hyperparameter corresponding to the length-scale of the RBF kernel. 
Mercer kernels are essential for Gaussian Processes (GPs) as they ensure valid covariance matrices and enable implicit high-dimensional data mapping. Additionally, they offer computational benefits through the ``kernel trick" in expansive spaces. Therefore, we analyze if the proposed kernel is a valid Mercer kernel.
\begin{theorem*}
    GSG is a valid Mercer Kernel for GP.
    \vspace{-2mm}
\end{theorem*}
\begin{proof}\small
The kernel in Equation \ref{eq:kernel} can be transformed as follows:
\begin{align*}
    &\mathcal{K}(x,x';\mathit{l}) \\
    &=exp(-\frac{||U^\top x-U^\top x'||^2}{2\mathit{l}^2})\\
    &=exp(-\frac{[(U^\top x)^\top U^\top x+(U^\top x')^\top U^\top x' - 2(U^\top x)^\top U^\top x']}{2\mathit{l}^2})\\
    &=exp(-\frac{[x^\top x+x'^\top x' - 2 x^\top x']}{2\mathit{l}^2})=exp(-\frac{||x-x'||^2}{2\mathit{l}^2}).
\end{align*}
Hence, $\mathcal{K}(x,x';\mathit{l})$ can be considered equivalent to the RBF kernel, which is widely recognized as a valid Mercer kernel. 
\end{proof}
\vspace{-2mm}
Next, we set up a Gaussian process (GP) with GSG kernel to realize Equation \ref{eq:obj}. 
This GP aims to estimate the existence probability of the provided source by evaluating the similarity between its corresponding and real observation.
\vspace{-2mm}
\begin{equation}\small
    GP: s \rightarrow \tau(o, o^*),
    \vspace{-2mm}
\end{equation}where $o$ is one observation by simulation from $s$.
\vspace{-2mm}

\subsubsection{Data acquisition.}

The surrogate GP needs to be initialized and trained iteratively, which both resort to sampling techniques. Initialization requires sampling multiple data points, while each iteration selects another data point from a new set of samples by maximizing an acquisition function, which uses the GP posterior to balance exploration and exploitation. 
Due to the discrete property of the graph data, traditional sampling methods, such as the Sobol sequence~\cite{sobol1967distribution}, do not fit the source localization problem. As a replacement, we propose a graph stratified sampling (GSS), which clusters the candidate and sample uniformly from each group. Specifically, GSS performs clustering over graph Fourier signals of candidate sources (Equation \ref{eq:fourier_transform}), and samples equal-size candidates from each cluster.

\begin{theorem*}
    GSS has lower variance than random sampling.
    \vspace{-2mm}
\end{theorem*}

\begin{proof}
Simple random sampling randomly draws \(  m \) samples from the entire population. The variance of its mean estimator is:
\vspace{-2mm}
{
\begin{align*}
\text{Var}(\bar{Y}_{\text{rs}}) &= \text{Var}\left(\frac{\sum_{i=1}^{m}Y_i}{m}\right)
=\frac{1}{m^2}\text{Var}\left(\sum_{i=1}^{m}Y_i\right) =\frac{\sigma^2}{m},
\end{align*}}where \( \sigma^2 = \text{Var} (Y_i)\) is the population variance. 
To set up GSS, we divide all candidates into \( \kappa \) non-overlapping equal-sized groups based on similarity. 
\( N \) is the population, and \( N_i \) is the population in \( i \)-th group.
From the \( i^{th} \) group, \( m_i \) samples are drawn, with a total of \( m = m_1 + m_2 + \dots + m_{\kappa} \) samples. The variance of this GSS mean estimator is given by:
$
\text{Var}(\bar{Y}_{\text{gss}})=\text{Var}\left(\sum_{i=i}^{\kappa}\bar{Y}_{i}\right) = \sum_{i=1}^{\kappa} \left( \frac{N_i}{N} \right)^2 \frac{\sigma_i^2}{m_i},
$
where \( \sigma_i \) is the sample mean of the \( i^{th} \) group.
To demonstrate the variance reduction of GSS compared to simple random sampling, we compare \( \text{Var}(\bar{Y}_{\text{gss}}) \) and \( \text{Var}(\bar{Y}_{\text{rs}}) \). Note that the within-group similarity exists, so the variances within each group are smaller than the overall population variance, i.e., \( \forall i, \sigma^2\geq\sigma^2_{i} \). In addition, the sample size of each group is the same (i.e., \( m_1=m_2=\ldots=m_c=\Tilde{m}, \text{and } \kappa\cdot \Tilde{m}=m \)), the size of each group is the same (\( \frac{N}{N_i}=\kappa \)). So:
\begin{equation*}
\text{Var}(\bar{Y}_{\text{gss}})= \sum_{i=1}^{\kappa} \left( \frac{1}{\kappa} \right)^2 \frac{\sigma_i^2}{\Tilde{m}} \leq \sum_{i=1}^{\kappa} \left( \frac{1}{\kappa} \right)^2 \frac{\sigma^2}{\Tilde{m}} = \text{Var}(\bar{Y}_{\text{rs}}).
\end{equation*}
\end{proof}
\vspace{-4mm}
GSS clusters similar items within each group, thereby reducing within-group variance and, consequently, the estimator's overall variance. This would aid Bayesian Optimization in minimizing overall variance and drawing precise conclusions about actual sources. 
Note that the expected sample mean by GSS is identical to the sample mean by random sampling, which is the population mean. Consequently, a decrease in variance reduces inference errors.

Expected improvement (EI) is used to estimate the potential improvement of samples over the current best observation. 
Suppose the model clusters all candidates into $b$ groups $C=\{c_1,c_2,\ldots,c_b\}$, $\gamma$ sets are sampled from each group such that $\{s_{ij}\}_{j=1}^{\gamma}\sim c_i$.
We optimize EI over the sample set $[s_{11},s_{12},...,s_{b\gamma}]$ such that:
\begin{equation*}\small
    \Tilde{s}^* = \argmax_{\Tilde{s}_{ij} \in [\Tilde{s}_{11},\Tilde{s}_{12},...,\Tilde{s}_{b\gamma}]} \textsc{EI}(\Tilde{s}_{ij}) = \argmax_{\Tilde{s}_{ij}} \mathbb{E}[\delta(s_{ij}, s+)\cdot I(s_{ij})],
\end{equation*} where $s+$ is the best set so far, $s_{ij}$ is the node set that corresponds to the graph Fourier transform signal $\Tilde{s}$, 
and $\delta(s, s+) = f(s;o^*) - f(s+;o^*)$.
$I(s_{ij})$ is an indicator function that equals to $1$ when $f(s_{ij};o^*)>f(s+;o^*)$ and $0$ when otherwise.
Although the search space in our problem is finite, enumerating all node sets in each iteration violates our principle of efficiency. Thus, we strategically sample a few sets with GSS and use EI to pick the maximizer. 

For the initial node sets and the one node set in each iteration, we need to query the true value of the objective function $\tau(o, o^*)$. The evaluation is achieved by simulations based on the given diffusion models, such as SI, SIS, or SIR. The proposed method~\algo~does not require the diffusion step as one of the inputs. Instead, our algorithm finds a simulation step $t$ that maximizes the similarity between the diffusion spread from $s$ and the given observation $o^*$ in each simulation round. The similarity $\textsc{SIM}$, an integer evaluated by the Hamming distance between the two vectors, keeps updating as the time step grows. We expect the similarity to grow first as diffusion time step $t$ grows. It peaks after a few steps and starts to decrease. On the one hand, with diffusion models like SI, the similarity decreases as the simulated spread suppresses the observed snapshot. On the other hand, with diffusion models like SIR and SIS, the similarity decreases when the diffusion waves of the simulation stagger the infected sub-graph of the observation. The simulation stops when the similarity shows a monotonically decreasing pattern and $\textsc{SIM}$ is set to its historically high. After multiple rounds of simulations, we have:
\vspace{-1mm}
\begin{equation}\small
    \tau(o, o^*) = \mathbb{E}[\max_t \textsc{SIM}(o_t, o^*;G,d,t,s)].
    \label{eq:tau}
    \vspace{-2mm}
\end{equation}

\subsection{Algorithm}

\algo~is demonstrated in Algorithm~\ref{alg:1}. Initiated with graph $G$ with $n$ source nodes, one-shot observation $o^*$, given diffusion model $d$, budget $k$, and sample size $\gamma$, it aims to produce an $n$-sized node set $s$ that approximates the true diffusion source. The algorithm selects the top $a$ nodes as the candidate pool based on degree centrality. A graph Fourier transform is applied on all $n$-sized subsets of $s^{pool}$ (lines \ref{alg:GFT_start}-\ref{alg:GFT_end}). These transformed sets are clustered into $c$ groups for later stratified sampling (line \ref{alg:GSS}). One graph Fourier transform signal is randomly sampled from each cluster and evaluated by the peak similarity between the diffusion spread and the observation achieved during the simulations as discussed in the data acquisition section (line \ref{alg:sim}). The $c$ pairs of Fourier representation of sources and similarities are used to train the GP model as an initialization (line \ref{alg:init_start}-\ref{alg:init_end}). In each following iteration, a new group of data points is sampled by GSS, and one of them is picked by the EI acquisition function. After evaluation, the GP model is updated with the new signal-similarity pair, and the process repeats until convergence or the iteration budget is used up (line \ref{alg:iter_start}-\ref{alg:iter_end}). After that, all candidate sets are evaluated with the model, and the maximizer is the estimated source set $\hat{s}$ (line \ref{alg:eval}).
\vspace{-1mm}
\begin{algorithm}[t]\small
\caption{\textbf{\algo}}\label{alg:1}
 \textbf{Input:} Graph $G$, source number $n$, observed snapshot $o^*$, diffusion model $d$, budget $k$, sample size $\gamma$\\
 \textbf{Output:} A $n$-sized node set $\hat{s}$
\begin{algorithmic}[1]
\STATE  set $\Tilde{S}\leftarrow \emptyset$, set $\Phi\leftarrow \emptyset$,  simulation step $t\leftarrow 0$ 
\STATE $s^{pool}$ $\leftarrow$ top $a$ nodes by degree centrality            \label{alg:preselect}
\STATE $S\leftarrow$ all $n$-size node sets $\subset$ $s^{pool}$                   \label{alg:GFT_start}
\FOR{$s\in S$ }                                                           \label{alg:loop_start}
    \STATE $\Tilde{s}\leftarrow$  $U^{\top}s$ as in Eq.~\ref{eq:fourier_transform}
    \STATE $\Tilde{S} \leftarrow  \Tilde{S} + \Tilde{s}$ 
\ENDFOR                                                                    \label{alg:GFT_end}
\STATE cluster $\Tilde{S}$ into $b$ groups: $C=\{c_1,c_2,\ldots,c_b\}$     \label{alg:GSS}
\STATE sample 1 sets from each set group, $\{\Tilde{s}_{i}\}_{i=1}^{b}\sim c_i$
\label{alg:init_start}
\FOR{$\Tilde{s}_{i} \in [\Tilde{s}_{1},\Tilde{s}_{2},...,\Tilde{s}_{b}] $}
    \STATE $s_i \leftarrow S[loc(\Tilde{s}_i)]$ where $loc(\cdot)$ is the index of $\cdot$ in $\Tilde{S}$ 
    \STATE $o_{t}\leftarrow$ simulate $d$ on $G$ with $s_{i}$ and increasing $t$  \label{alg:sim}
    \STATE $\tau_{i} \leftarrow \mathbb{E}[\max_t$ \textsc{SIM}($o^*$, $o_{t}$)] as in Eq.~\ref{eq:tau}
    \STATE $\Phi \leftarrow \Phi + (\Tilde{s}_{i}, \tau_{i})$ 
\ENDFOR
\STATE train GP (as surrogate) with $\Phi$: $\Tilde{s}\xrightarrow{GP}\tau$        \label{alg:init_end}
\WHILE{$z\neq$0 ($z=k-b$)}                                                 \label{alg:iter_start}
    \STATE sample $\gamma$ sets from each set group, $\{\Tilde{s}_{ij}\}_{j=1}^{\gamma}\sim c_b$
    \STATE $\Tilde{s}^* \leftarrow \arg\max_{\Tilde{s}_{ij}} \text{EI}(\Tilde{s}_{ij})$, and $\Tilde{s}_{ij} \in [\Tilde{s}_{11},\Tilde{s}_{12},...,\Tilde{s}_{b\gamma}] $
    \STATE $s^* \leftarrow S[loc(\Tilde{s}^*)]$
    \STATE $o_{t}\leftarrow$ simulate $d$ on $G$ with $s$ and increasing $t$
    \STATE $\tau^* \leftarrow \mathbb{E}[\max_t$ \textsc{SIM}($o^*$, $o_{t}$)]
    \STATE $\Phi \leftarrow \Phi + (\Tilde{s}^*, \tau^*)$ and re-train GP with $\Phi$
    \STATE $z \leftarrow z - 1$
\ENDWHILE                                                                  \label{alg:iter_end}
\STATE Evaluate $S$ with GP: ${\hat{s}}=\arg\max_{s \in S}\text{GP}(s)$   \label{alg:eval}
\end{algorithmic}
\end{algorithm}

\subsection{Time complexity}

\begin{table*}[htb]
\small
\centering
\resizebox{0.9\textwidth}{!}{%
\begin{tabular}{lllllll}
\hline
Methods         & Jordan Centrality & LPSI     & NetSleuth    & LISN          & BOSI\_prep & BOSI\_opti \\
\hline
Time Complexity & $\mathcal{O}(N^3)$          & $\mathcal{O}(N^3)$ & $\mathcal{O}(|V_I|+|E_I|+|E|)$ &   $\mathcal{O}(N^3)$ & $\mathcal{O}(N^3)$         & $\mathcal{O}(k^4)$          \\
\hline
\end{tabular}%
}
\vspace{-2mm}
\caption{Time complexities of~\algo~and other popular alternatives.}
\label{tab:complexity}
\vspace*{-6mm}
\end{table*}
We analyze the time complexity of~\algo~based on Algorithm~\ref{alg:1} and compare it with popular multi-source localization methods. Selecting $a$ nodes with the highest degree centralities (line~\ref{alg:preselect}) is $\mathcal{O}(|V|+|E|) = \mathcal{O}(N^2)$ using BFS traversal. This complexity can be further reduced to $\mathcal{O}(N)$ for sparse graphs. Calculating the graph Fourier transform operator is $\mathcal{O}(N^3)$~\cite{merris1994laplacian}. Generating all $n$-sized node sets from $s$ (line~\ref{alg:GFT_start}) requires $\mathcal{O}(a^n)$. Looping through all combinations has a time complexity of $\mathcal{O}(a^n)$, and the operations inside the loop are multiplications between $1*N$ vectors and $N*N$ matrices, which are $\mathcal{O}(N^2)$. Thus, the time complexity for the whole block (line~\ref{alg:loop_start}-\ref{alg:GFT_end}) is $\mathcal{O}(a^n N^2)$. Clustering the graph Fourier signals is $\mathcal{O}(a^n)$. The operations above are the preparations for the GP training and have a combined time complexity of
\vspace{-2mm}
\begin{equation}\small
    \mathcal{O}(N^2 + N^3 + a^n + a^n N^2 + a^n) = \mathcal{O}(N^3),
\vspace{-2mm}
\end{equation}
when $a \ll N$ and $n$ is a very small integer.

Breaking down the GP training process and the final prediction (line \ref{alg:init_start} - \ref{alg:eval}), we get all the operations carried out. There are $(b+b\gamma(k-b))$ samplings from the cluster, $k$ simulations to evaluate the true similarity between the simulated spread and the true observation, $(k-b+1)$ rounds of GP model training, and $1$ evaluation for each candidate node sets using the trained GP. Assuming each simulation takes a long but constant time, the time complexities of sampling, simulation, GP training, and evaluation are $\mathcal{O}(1), \mathcal{O}(1), \mathcal{O}(|\Phi|^3), \text{and } \mathcal{O}(1)$~\cite{rasmussen2006gaussian}, respectively. Thus, the time complexity for the whole training and predicting period is
\vspace{-2mm}
\begin{equation}\small
    \mathcal{O}(b+bk\gamma-b^2\gamma+k+(k-b+1)|\Phi|^3+a^n) = \mathcal{O}(k^4)
\vspace{-2mm}
\end{equation} since $|\Phi| \leq k$, $b$ and $g$ are constants, and $a^n$ is ignorable compared to the problem size $N$.
The overall time complexity of~\algo~is $\mathcal{O}(N^3+k^4)$ where $N$ is the graph size and $k$ is the evaluation budget. This complexity is compared with other methods in Table~\ref{tab:complexity}. We can see that NetSleuth, proposed as an efficient algorithm, still has a better scalability. But~\algo, as a simulation-based method, has the same time complexity as the other heuristics when the budget is relatively small. Also, note that the only operation bonded with the $\mathcal{O}(N^3)$, namely the eigendecomposition of the Laplacian matrix, runs only once in~\algo; we expect its running time to grow slower than the other baselines.
\vspace{-2mm}

\section{Experiment}
\label{exp}

The experiments are carried out with 32 AMD EPYC 7302P 16-Core processors and 32GB RAM. Simulations are performed by NDLib~\cite{rossetti2018ndlib}, an open-source toolkit for diffusion dynamics. The baselines are realized by Cosasi~\cite{mccabe2022cosasi}, a Python package for graph diffusion source localization. The Bayesian optimization paradigm is implemented by BOTorch~\cite{balandat2020botorch} and gPyTorch~\cite{gardner2018gpytorch}.
\vspace{-2mm}
\subsection{Configurations}
We adopt SI, SIR, SIS, and IC as diffusion models. The infection rate is set to be $0.1$ in the epidemic models, and the recovery rate in SIR and SIS is set to be $0.1$. Each candidate source set is evaluated by an average of $100$ simulation rounds. The Bayesian optimization paradigm includes $50$ iterations to train the final model.
\textbf{Datasets:}
Three real-world datasets, namely Cora, CiteSeer, and PubMed~\cite{yang2016revisiting}, reproduce the complex social network structure. Since the source localization problem is traditionally studied on connected graphs, we take the largest connected component of these graphs as the studied network. Two synthetic graphs are generated using NetworkX to represent pseudo social networks. They include \textit{connected Watts-Strogatz small-world graphs} (SW)~\cite{watts1998collective} and {\textit{Erdős–Rényi random graphs}} (ER)~\cite{gilbert1959random}. Each synthetic graph has $1,000$ nodes for effectiveness evaluation, and the average degree is around $10$. We use SW graphs with sizes ranging from $1,000$ to $5,000$ for runtime analysis.
\textbf{Baselines:}
\algo~is compared to three popular baselines. (1) Jordan centrality (JC)~\cite{shah2010detecting} greedily selects the nodes with the smallest maximum distances to other nodes in the infected sub-graph.
(2) NetSleuth (Net)~\cite{prakash2012spotting} is a highly efficient algorithm to identify the number and the location of sources under the SI model.
(3) LISN~\cite{nie2019localizing} scores nodes based on shortest distance and maximum likelihood and selects nodes with the highest scores.
\textbf{Metrics:}
The result is evaluated by its distance from the true sources. In our experiment, the identified and true source sets are the same size. Thus, the distance between the two sets is calculated by
$\mathcal{D}\{a, b\} = \min \{\Sigma_i\Delta{(a_i, \hat{b}_i)}\}$,
where $\Delta{(a_i, \hat{b}_i)}$ represents the shortest distance between the nodes $a_i$ and $b_i$, and $\hat{b}$ stands for a permutation of list $b$.
\vspace{-2mm}

\subsection{Results}
The empirical study includes 
(1) Performance: the effectiveness of~\algo~is compared against the baselines; 
(2) Runtime Analysis: the time cost of~\algo~and the baselines are demonstrated to verify the time complexity analysis; and 
(3) Ablation Test: we compare~\algo~with two variants to evaluate the utility of our proposed GSG and GSS.
\vspace{-2mm}
\begin{figure}[b]
    \centering
    \vspace*{-4mm}
    \includegraphics[width=0.96\columnwidth]{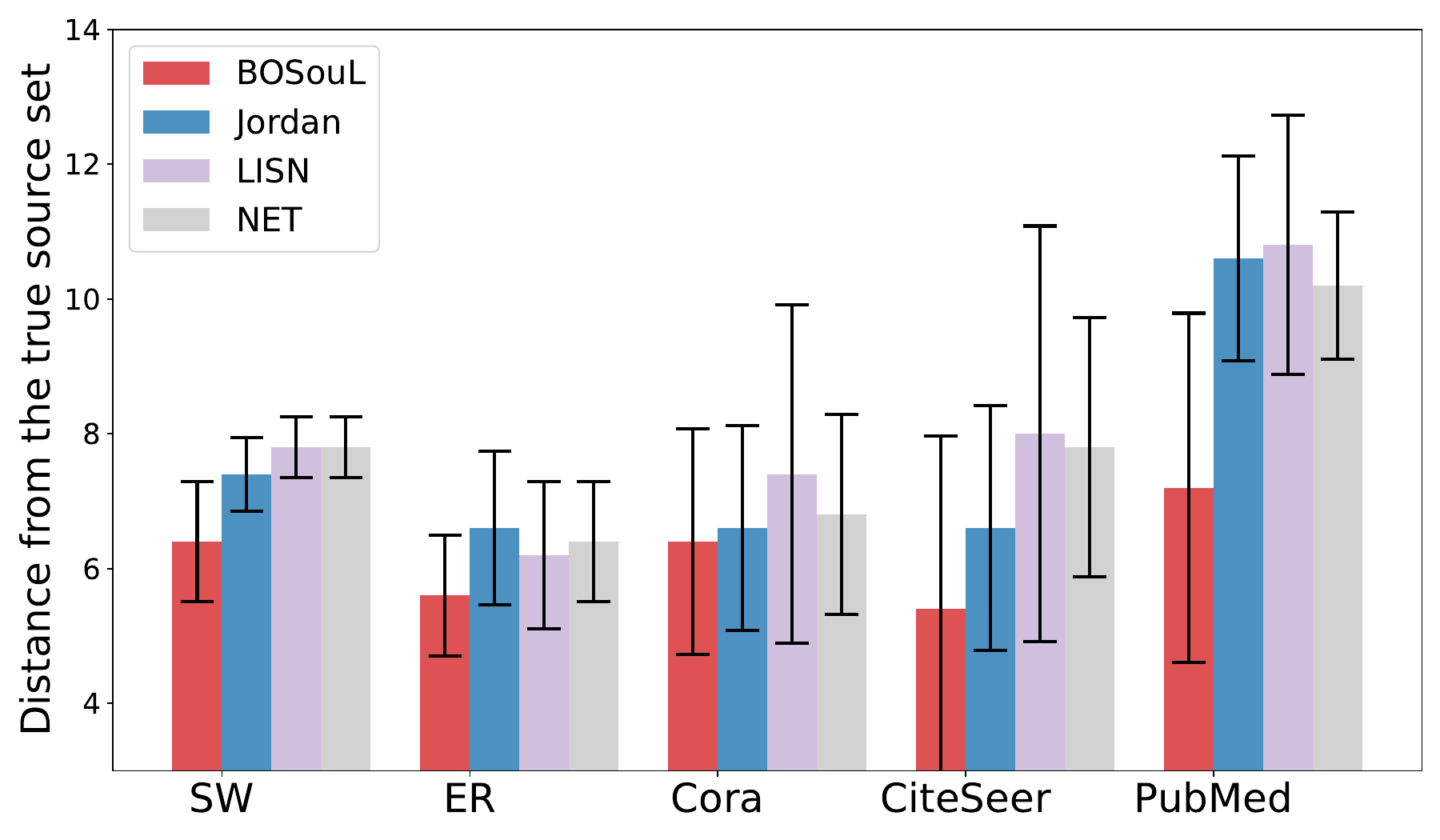}
    \vspace{-4mm}
    \caption{The distance toward the true source set with SIR.}
    \label{fig:SIR}
\end{figure}
\begin{figure}
    \centering
    \includegraphics[width=0.96\columnwidth]{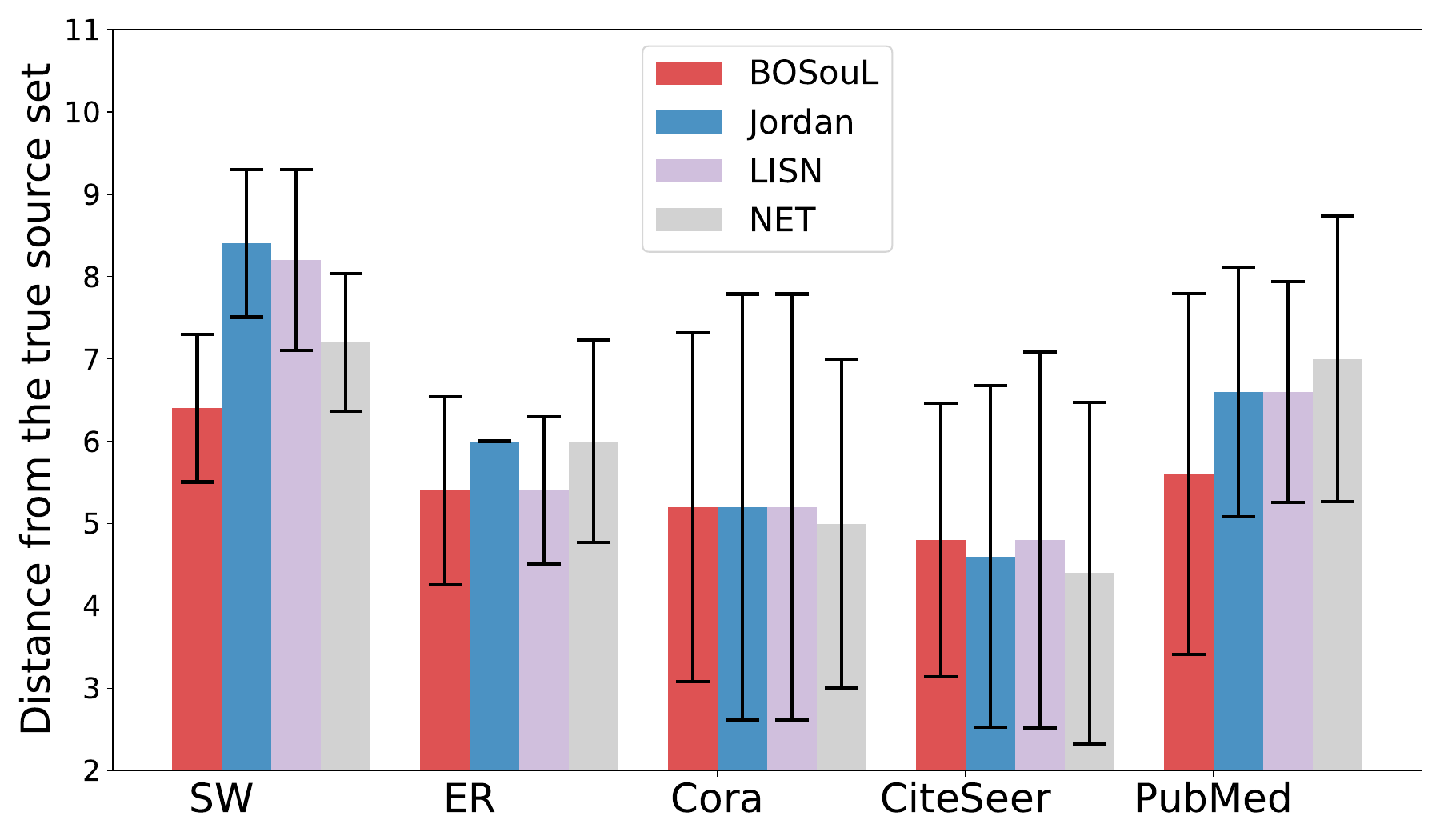}
    \vspace{-4mm}
    \caption{The distance toward the true source set with SI.}
    \label{fig:SI}
    \vspace*{-6mm}
\end{figure}
\subsubsection{Performance.}
To demonstrate the compatibility of~\algo, we test it along with the baselines with two significantly different diffusion models, SIR and SI. The infection rate for both models is $0.1$, and the recovery rate for SIR is $0.1$. For~\algo, the budget of simulation is $70$. The number of candidate nodes is $50$, thus there are $\binom{50}{3} = 19,600$ candidate sets. They are clustered into $20$ groups. Each method runs $10$ times on each graph with different true source sets of size $3$. The mean and standard deviation are reported as the final results.
In SIR models where the source nodes might already recover and do not distinguish from the nodes that have never been infected, identifying the source nodes is much more challenging. Figure~\ref{fig:SIR} clearly illustrates the superiority of~\algo~when solving multi-source localization problems with the SIR model. It achieves the lowest localization error on all five datasets, demonstrating its effectiveness and diffusion model adaptability. Compared to other methods, the performance advance is most significant on SW, CiteSeer, and PubMed, with an enhancement of up to 29\%.
Understandably, the baseline methods do not perform well enough with the SIR model since they explicitly or implicitly assume the SI model. Thus, they only select nodes in the infected sub-graph. Figure \ref{fig:SI} demonstrates that ~\algo~also achieves competitive performance on all five graphs with the SI model. It outperforms the other four methods on four datasets. On CiteSeer, NetSleuth is the best performer but only leads~\algo~by 0.4. 
Additional experiments show that~\algo~surpasses all baselines across all five graphs under the IC model and outperforms baselines on four datasets except for Cora with the SIS model.
\vspace{-2mm}


\subsubsection{Runtime Analysis.}

\begin{figure}[b]
    \centering
    \vspace{-6mm}
    \includegraphics[width=\columnwidth]{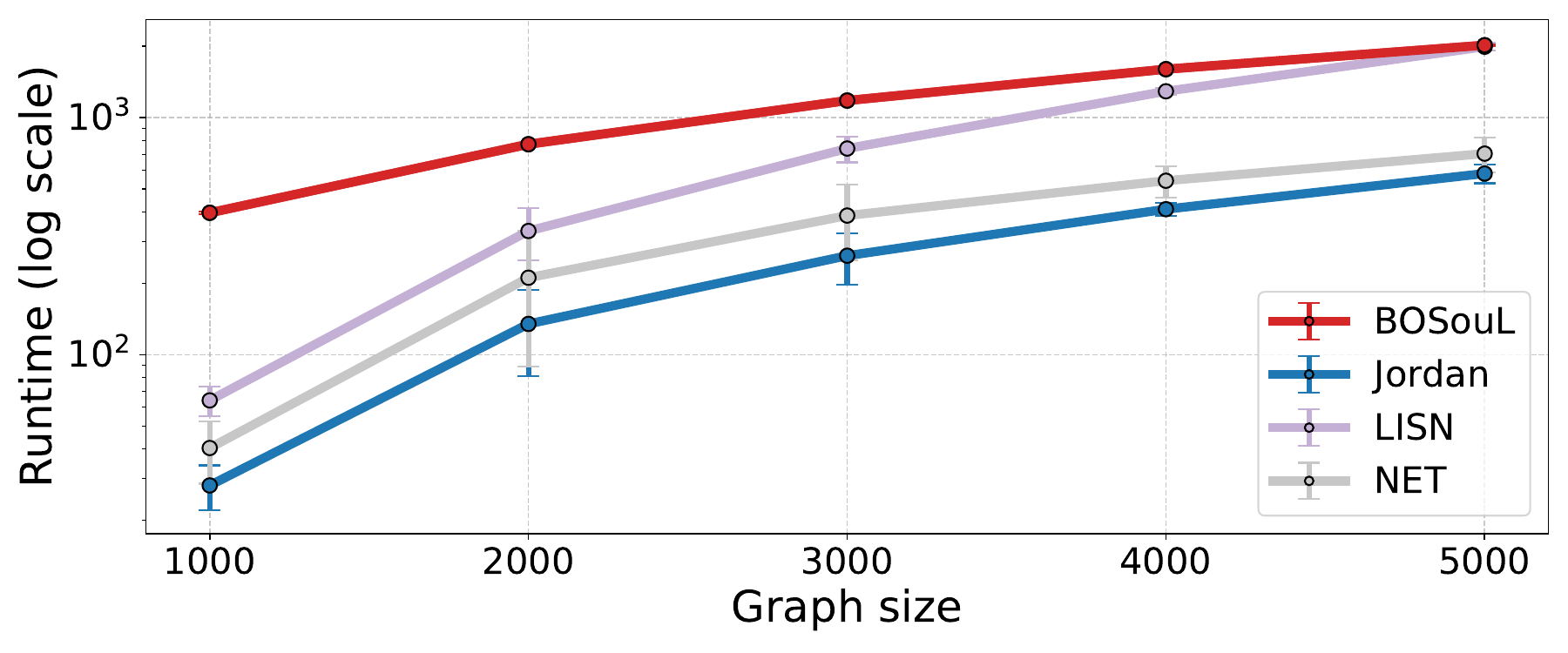}
    \vspace{-6mm}
    \caption{Runtime analysis on SW graph.}
    \label{fig:time}
\end{figure}
As expected,~\algo~has longer running times than the baseline methods on SW graphs with increasing sizes due to the time spent on simulations. This difference is most significant when graph size $N = 1,000$. ~\algo~takes $396.69$ seconds, about six times slower than the slowest baseline LISN. Comparatively, Jordan centrality only takes $28.10$ seconds, and NetSleuth needs $40.42$ seconds. Those are $7.08\%$ and $10.19\%$ of~\algo's running time, respectively. But this difference shrinks as the graph size grows.~\algo~almost scales linearly in the empirical experiment, evidenced by the steady increase in mean runtime. Jordan centrality spends the most time on eigendecomposition, which is part of~\algo. Thus, it shows a similar pattern with a slightly higher rate of increase. NetSleuth, despite its lower time complexity, consistently takes more time than Jordan centrality as the graph size grows from $1,000$ to $5,000$. Also, it has a relatively large variance due to the term representing the size of the infected subgraph in its time complexity. Lastly, LISN's running time grows fastest as the problem size scales because it involves several matrix multiplications. On a size-$5,000$ connected small-world graph,~\algo~takes $2,018.98$ seconds, while LISN takes $1,989.67$ seconds. We can expect that on a larger graph, the running time of the latter will surpass that of the former. The running times for Jordan centrality and NetSleuth are $580.86$ seconds and $704.72$ seconds, respectively. The percentages compared to~\algo~raise to $28.77\%$ and $34.90\%$. This trend is illustrated in Figure~\ref{fig:time}.
Overall,~\algo~shows a stable scalability. The adaptation of Bayesian optimization makes the simulation-based approach tractable. Although its running time is longer than the faster baseline methods like Jordan centrality and NetSleuth, it remains competitive due to its superior performance. At the same time, the efficiency gap shrinks as the problem size grows.

\begin{table}[t]\small

\begin{tabular}{l|l|l|l|l|l}
\hline
   & SW                        & ER                      & Cora                    & CiteSeer                & PubMed                  \\ \hline
raw  & 7.4$\pm$0.6            & 6.0$\pm$1.2         & 8.6$\pm$2.3          & 7.6$\pm$2.5          & 7.6$\pm$2.3          \\ \hline
(GSG)\label{col:gsg}&  -10.8\% & -3.3\% & -23.2\% & -18.4\%  & 0\%        \\ \hline
half  & 6.6$\pm$1.5            & 5.8$\pm$0.5          & 6.6$\pm$1.1          & 6.2$\pm$1.8          & 7.6$\pm$2.2          \\ \hline
(GSS)\label{col:gss}& -2.7\%   & -3.3\% & -2.4\% & -10.5\% & -5.2\%\\ \hline
full & 6.4$\pm$0.9   & 5.6$\pm$0.9 & 6.4$\pm$1.7 & 5.4$\pm$2.7 & 7.2$\pm$2.6 \\ \hline
\end{tabular}%
\vspace{-2mm}
\caption{Ablation tests with SIR model.}
\label{tab:ablation}
\vspace*{-6mm}
\end{table}

\subsubsection{Ablation Study.}
Table~\ref{tab:ablation} compares the performance of our proposed~\algo~method (full) against two ablated versions: using random sampling (RS) instead of GSS for data acquisition and using an RBF kernel instead of the proposed GSG kernel. Column (GSG) shows the percentage decrease in localization error after substituting the RBF kernel with the GSG kernel. And Column (GSS) demonstrates the further performance increase brought by GSS.
We can observe that RS+GSG (half) always performs at least as well as RS+RBF (raw). More specifically, except for PubMed, GSG brings performance enhancement ranging from $0.2$ to $2.0$, which is a $3.3\%-23.2\%$ decrease in the localization error.
This shows the benefits of the graph spectral Gaussian kernel for effective adaptation to the graph-structured data and the source localization problem.
It is also demonstrated that GSS+GSG (full) outperforms RS+GSG on all five datasets.
This shows the benefits of graph stratified sampling for uniform data acquisition. It explores the search space better than random sampling.
In sum, our ablation study verifies the proposed components each provides significant gains over variants without those techniques. GSG kernel consistently assists in graph-structured data adaptation, fulfilling the smoothness assumption. Graph stratified sampling is crucial for handling more complex search spaces.
\vspace{-3mm}

\section{Conclusion}
This study presents a simulation-based method~\algo~for multi-source localization from a one-shot observation. Bayesian optimization is adopted to foster efficiency and reveal a relationship between the node set and the observation. We theoretically prove that GSG, a graph-level kernel for the Gaussian process, is a valid Mercer kernel, and GSS, a stratified sampling method based on graph clustering, reduces variance better than random sampling.

\newpage

\section*{Acknowledgements} 
This work is supported by NSF IIS award \#2153369.

\bibliography{aaai24}

\end{document}